\documentclass{article} 

\usepackage{graphicx}
\usepackage{amsmath}
\usepackage{framed}
\usepackage{amssymb,amsthm}
\usepackage{comment}
\usepackage{ bbold }
\usepackage{xcolor}
\usepackage{url}

\newtheorem{theorem}{Theorem}

\begin{document}

\title{Online Multi-Agent Decentralized Byzantine-robust Gradient Estimation}
%
%
\author{Alexandre Reiffers-Masson,
Isabel Amigo}
%

%

\maketitle              

\begin{abstract}
In this paper, we propose an iterative scheme for distributed Byzantine-
resilient estimation of a gradient associated with a black-box model. Our algorithm is based on simultaneous perturbation, secure state estimation and two-timescale stochastic approximations. We also show the performance of our algorithm through numerical experiments. 
\end{abstract}
\section{Introduction}

The main goal of this paper is to derive a decentralized algorithm which can efficiently learn the gradient of a black-box model, in a multi-agent context. In a black-box model it is assumed that a function  $f$ is unknown but can be accessed through queries to a zero-th order oracle \cite{bubeck2015convex}. Being able to compute the gradient can be used, for instance to design efficient distributed optimization algorithms to find the minimum of $f$. We assume that there is a finite number of processors/servers (called nodes or agents in the rest of the paper) which participate in the distributed computation of the gradient. 

We also assume that some agents can have Byzantine behaviors: that is, they will try to deviate from the suggested protocol. Such behaviors are well known in the literature of distributed algorithms (consensus and leader election algorithms, for instance) and have also recently been studied in the context of machine learning \cite{chen2017distributed,yin2018byzantine,wu2020federated}. In such contexts, three major points need to be tackled: (1) The fact that unidentifiable Byzantines nodes are present in the system and that it is not possible to guarantee the safety of a given node; (2) the algorithms should be decentralized and each "good" node, should have a good estimation of the gradient (3) the algorithms should be efficient (in terms of speed of convergence and number of operations per iterations). 

Our solution is based on a generalization of the gradient estimator algorithm developed in \cite{borkar2017gradient}, using two-timescale stochastic approximations and secure estimation \cite{fawzi2011secure}.  We are also able to derive sufficient conditions to ensure that our algorithm will be able to be robust to Byzantine nodes. 






\section{Problem formulation}\label{sec:problem formulation}

This section is dedicated to the description of our framework and to the problem formulation. Further discussion on the model assumptions is provided in Subsection \ref{subsec:remarks}.

\paragraph{Nodes' model} Let $\mathcal{I}=\{1,\ldots,n\}$ be the set of nodes. We assume that the set of nodes can be partitioned into two disjoint subsets of nodes $\mathcal{I}_B$ and $\mathcal{I}_G$. The set $\mathcal{I}_B$ (resp. $\mathcal{I}_G$) is the set of Byzantine nodes (resp. non-Byzantine nodes).  At every instant $k\in\mathbb{N}$, each node $i \in \mathcal{I}$ controls a variable  $x_i(k)\in \mathbb{R}$. For every $k\in\mathbb{N}$, let $x(k):=[x_i(k)]_{1\leq i \leq n}$ be the associated vector. A Byzantine node will try to disobey the protocol. More precisely, we assume that: (1) a Byzantine node can modify its own $x_i(k)$ (when it is allowed) without following the prescribed protocol (2) it can report corrupted values (3) it can have complete knowledge of the system and the algorithm. Note that $\mathcal{I}_B$ is constant for every $k$, meaning that the set of Byzantine nodes does not change over time (this assumption is similar to the one in \cite{fawzi2011secure}).

\paragraph{Asynchronous Updates}
Let $Y(k)\subseteq \mathcal{I}$ be the set of active nodes at instant $k\in\mathbb{N}_+$. We say that a node $i$ is active at instant $k$ if it can modify $x_i$ at instant $k$.  Let $Y_B(k)\subseteq \mathcal{I}_B$ be the set of active Byzantine nodes and $Y_G(k)\subseteq \mathcal{I}_G$ is the set of active non-Byzantine nodes, such that $Y(k)=Y_B(k) \cup Y_G(k)$ and $Y_B(k)\cap Y_G(k)=\emptyset$ . We assume that $Y(k)$ is an irreducible Markov chain where the state space is a subset $\mathcal{U}$ of the power set $\mathcal{P}(\mathcal{I})$ of $\mathcal{I}$. This assumption ensures that we are in an asynchronous set-up, where not all nodes are active at the same time. This is crucial to ensure that Byzantine nodes do not have too much power in the system. We use a vector notation $u(k)\in\{0,1\}^{\mid\mathcal{I}\mid}$ to denote an element of $\mathcal{P}(\mathcal{I})$ at instant $k$. If the $i$-th component of $u(k)$ is equal to one, it means that node $i$ belongs to the set $Y(k)$ (i.e. it is active). Otherwise, node $i$ is not in the set $Y(k)$ (i.e. it is not active). Note that for a given $k$, we have $u(k)=[\mathbf{1}_{i\in Y(k)}]_{1\leq i \leq n}$. The subset of $\mathcal{U}$ where each vector contains at least a Byzantine node is denoted by $\mathcal{U}_B:=\{u\in\mathcal{U}\mid \exists i\in\mathcal{I}_B,\;u_i=1\}$.

\paragraph{Challenge} We assume that for every $k$, the nodes are interested in computing the gradient of $f(x(k))$ where $f:\mathbb{R}^n\rightarrow \mathbb{R}$ is a continuous twice differentiable function. We assume that for every $x(k)$, $f(x(k))$ is available to every node. However, $\nabla f(x(k))$ is not readily available, and therefore, the nodes need to collaborate together to obtain a robust estimation. We can easily generalize our approach to the case when $f:\mathbb{R}^n\rightarrow \mathbb{R}^{p}$, but for simplicity's sake, we assume $p=1$. 

\subsection{Remarks}\label{subsec:remarks}
We believe that the assumption that $Y(k)$ is an irreducible Markov chain, resulting in mathematical simplicity in our problem, is in line with acceptable models capturing asynchronous updates of the nodes. This assumption is necessary for the proof of the convergence of iterative scheme (see section \ref{sec:algo_num}) We justify below our assumption with a classical example coming from the distributed systems literature. A first classical communication scheme to ensure lightweight communication costs is a simple token-passing implementation of a random walk. 
Different implementations of such mechanism have been studied in distributed systems \cite{augustine2013storage,ramiro2014temporal}
and Byzantine versions of such communication scheme have been studied in \cite{augustine2015fast, yuan2019fast}. The performance of the different algorithms has been studied using Markov chain models. Another example are wireless networks. In this context, comes naturally the constraint that all links/nodes cannot be  activated simultaneously. A natural protocol in such case is the CSMA protocol which has been modeled using Markov chains \cite{jiang2009distributed,borkar2014asynchronous}.
\section{Byzantine-robust Gradient Estimation}\label{sec:byzantine estimation}

We first describe a natural, not robust to Byzantine nodes, way to approximate the gradient. Then, by adapting the tools developed in \cite{fawzi2011secure}, we present how to construct a robust approximation of $\nabla f(x)$. 

\subsection{Simultaneous Perturbation} We now introduce the simultaneous perturbation scheme to create an estimate of the gradient. We will also illustrate the impact of having Byzantine nodes in such estimator. 

In a simultaneous perturbation scheme, at every instant $k$, for a given $x(k)$, we assume that every node $i\in Y_G(k)$ performs the following scheme:
\begin{enumerate}
    \item Sample $\Delta_i(k)\in\{-1,1\}$ from a Bernoulli distribution of parameter $1/2$,
    \item Play $x_i(k)+\delta \Delta_i(k)$, where $\delta$ is positive constant, provided as input.
\end{enumerate}
A  Byzantine node will not respect such steps. Without loss of generality, we assume that a Byzantine node, instead of playing  $x_i+\delta \Delta_i(k)$ plays $x_i+\delta e_i^1(k)$, with $e_i^1(k)\in\mathbb{R}$. Therefore, the observation by all nodes at instant $k$ is given by:
\begin{equation}
    z_{i,u(k)}(k)=\left\{\begin{array}{lll}
        \displaystyle \frac{f(x+\delta\tilde\Delta(k))-f(x)}{\delta\tilde\Delta_i(k)}, & \text{if} & i\in Y_G(k),\\
       0, & \text{if} & i \notin Y(k), \\
        \end{array}
    \right.
\end{equation}
where the $i$-th component of $\tilde\Delta(k)$ is equal to:
\begin{equation}
   \tilde\Delta_i(k)=\left\{\begin{array}{lc}
      \Delta_i(k),  &\text{if } i\in Y_G(k), \\
      e^1_i(k), & \text{if } i\in Y_B(k), \\
      0, & \text{if } i\notin Y(k). \\
   \end{array}\right.
\end{equation}
In expectation, and using Taylor's theorem (see chap. 10 p.120 in \cite{borkar2009stochastic}), the gradient observation,  $\overline{z}_{i,u(k)}(x):=E[z_{i,u(k)}(k)]$, can be summarized as follows:
\begin{equation}
    \overline{z}_{i,u(k)}(x):=\left\{\begin{array}{l}
        \displaystyle \frac{\partial f}{\partial x_i}(x)+O(\delta \|\nabla^2f(x)\|),  \\
        \;\;\;\;\;\;\;\;\text{ if }  i\in Y_G(k)\text{ and }\;Y_B(k)=\emptyset,\\\displaystyle
       \frac{\partial f}{\partial x_i}(x)+e_{u(k)}(k)+O(\delta \|\nabla^2f(x)\|),\\
        \;\;\;\;\;\;\;\;\text{ if }  i\in Y_G(k)\text{ and }\;Y_B(k)\neq\emptyset, \\
       0  \text{ if }  i \notin Y(k). \\
        \end{array}
    \right.
\end{equation}

Note that we have summarized the impact of having $i$ or $j$ being a Byzantine node by simply adding a perturbation term $e(k)\in\mathbb{R}$ to the observation. Indeed, without loss of generality, because the Byzantine node is not constrained in the choice of $e(k)$, such additive term captures the impact of the perturbation terms $e_i^1(k)$ and $e_i^2(k)$. From this remark, we can also see that $z_{i,u(k)}(k)$ is not a robust estimate of the gradient, due to the fact that the moment $Y_B(k)\neq \emptyset$, the observation $z_{i,u(k)}(k)$ can be any value in $\mathbb{R}$.

\subsection{Secure gradient estimation}
We will now construct a robust estimate of the gradient, in a distributed manner, using the theory developed in \cite{fawzi2011secure}. Let us first introduce the following assumption over $f(x)$.

\textbf{Assumption A:} It exists a function $v:\mathbb{R}^n\rightarrow \mathbb{R}^m$, such that the function $f:\mathbb{R}^n\rightarrow \mathbb{R}$ satisfies:
\begin{equation*}
    \nabla f(\mathbf{x})=Av(x),\;\forall x.
\end{equation*}
We assume that  $v:\mathbb{R}^n\rightarrow \mathbb{R}^m$  is a continuous function and $A$ is a $n\times m$ real matrix. We also assume that $A$ is known by every node. 

To recast our problem in the form described in \cite{fawzi2011secure}, we rewrite the observations as the output of a linear system. For a given $x$ the vector $\overline{z}(x):=[[\overline{z}_{i,u}(x)]_{1\leq i\leq I}]_{u\in\mathcal{U}}$, gives the gradient observation for the different $u$. The vector $v(x):=[v_{m'}(x)]_{1\leq m'\leq m}$ is the vector that we are interested in learning for every $x$. Note that the relationship between $\overline{z}(x)$ and $v(x)$ is captured by the following linear system $\overline{z}(x)=A_1v(x)+e$, where: the matrix $A_1$ is equal to:
\begin{equation*}
    A_1:=\begin{bmatrix}
 A(u_1) \\
  A(u_2) \\
 \vdots \\
A(u_{\mid \mathcal{U}\mid}) 
\end{bmatrix},
\end{equation*}
with $A(u)=u \otimes A$ where $\otimes$ denotes entrywise product. Using the entrywise product will ensure that the $i$-th line of $A(u)$ is equal to the null line vector if $u_i=0$, meaning that node $i$ is not active in this case. (2) The vector $e:=[e_u\mathbf{1}^T]_{u\in\mathcal{U}}$ captures the errors injected by the Byzantine nodes, where $e_u\neq 0$ if and only if $Y_B=\emptyset$ and where $\mathbf{1}$ is the all-ones vector of dimension $n$. 

Let us adapt the main result (proposition 6 of \cite{fawzi2011secure}), for the secure estimation of $v(x)$. Note that the next theorem is based on the mean of the observations and not the real observations. We then design an algorithm based on the observations which leverages the next theorem and ensures a safe estimation of the gradient.

\begin{theorem}\label{thm:fawzi}

For a given $x$, let us define solution $v^*(x)$ to be the solution of the following optimization problem:
\begin{equation}\label{eq:decoder}
    \min_{v\in\mathbb{R}^m}J(v):=\sum_{u\in\mathcal{U}}\sum_{i=1}^n|\overline{z}_{i,u}(x)-A_{i}(u)v|,
\end{equation}
where $A_{i}(u)$ is the $i$-th row of $A(u)$, $\overline{z}_{i}(x)=[\overline{z}_{i,u}(x)]_{u\in \mathcal{U}}$.
If, for all $\mathcal{K}\subset\{1,\ldots,|\mathcal{U}|n\}$, such that $|\mathcal{K}|=q$, with 
\begin{equation}\label{eq:nec_suf condition}
\sum_{k\in \mathcal{K}}|A_{k-n\lfloor k/n\rfloor+1}(u_{\lfloor k/n\rfloor+1})z|\leq\sum_{k\in \mathcal{K}^c}|A_{k-n\lfloor k/n\rfloor}(u_{\lfloor k/n\rfloor+1})z|,\end{equation} for all $z\in\mathbb{R}^{m}- \{0\}$, then, when $\delta \rightarrow 0$,  $v^*(x)=v(x)$ when $|\mathcal{U}_B|=\lfloor q/n \rfloor$.
\end{theorem}
\begin{proof}
Directly adapted from proposition 6 in \cite{fawzi2011secure}.
\end{proof}
The previous theorem indicates us that as long as \eqref{eq:nec_suf condition} is satisfied, we can use solution \eqref{eq:decoder} to retrieve the gradient. Simpler sufficient conditions have been mentioned in \cite{fawzi2011secure}.  

\section{Algorithm and numerical study}\label{sec:algo_num}

We now describe the algorithm for the non-Byzantine nodes. It can be summarized as follows: At every instant $k$, a subset of nodes perturbate the black-box model, obtaining an estimator of their respective partial derivative. Once the active nodes observe the perturbated function, they broadcast it to all the nodes. Using these observations, each node estimates the gradient based on \eqref{eq:decoder}.

\begin{framed}
\textbf{Decentralized Byzantine Gradient Estimation}

\textit{Initialization: } Input, $x$, the working point. \\ 
For each round $k=1,2,\ldots$:
\begin{itemize}
\item \textit{Perturbation Step:} For every node $i\in Y_G(k)$:
\begin{enumerate}
    \item Sample $\Delta_i(k)\in\{-1,1\}$ from a Bernoulli distribution of parameter $1/2$,
    \item Play $x_i+\delta \Delta_i(k)$.
    \item Observe $z_{i,u(k)}$ and $u(k)$ and broadcast $z_{i,u(k)}$ to every node.
\end{enumerate}
    \item \textit{Decoding step}: Update
    \begin{eqnarray*}
        \hat z_{i,u}^{k+1}&=&\hat z_{i,u}^k+b(k)\mathbb{1}_{u,u(k)}(z_{i,u(k)}-\hat z_{i,u}^k ),\;\forall i\in\mathcal{I},\forall u\in\mathcal{U},\\
      v^{k+1}&=&v^k+a(k) \sum_{u,i}A_{i}(u)^T\text{sign}(\hat{z}_{i,u}(k)-A_{i}(u)v^k),
    \end{eqnarray*}
    where $v^k:=[v_{m'}^k]_{1\leq m'\leq m}$, $a(k)$ and $b(k)$ are such that $\sum_k a(k)=\sum_k b(k)=\infty$, $\sum_k(a(k)^2+b(k)^2)<\infty$ and $\lim \frac{a(k)}{b(k)}\rightarrow 0$.
    \item \textit{Gradient estimator:} Use $A_1v^k$ as an estimator of $\nabla f(x)$. 
\end{itemize}
\end{framed}

The convergence of the algorithm towards  $v^*(x)$ comes from the fact that it can be seen as  an asynchronous two-time scale stochastic approximation, where on the fast time-scale every node runs a moving average on $\hat z_{i,u}(k)$ for every $k$ and every $i$ and on the slow time-scale, a gradient descent is used to solve, in an online fashion \eqref{eq:decoder}.  Convergence can be proved by using the  ordinary differential equation approach (see Ch. 6 and 7 in \cite{borkar2009stochastic}).

For our illustrative numerical study, we assume $f(x)=\frac{1}{C-\sum_i^n x_i}$, where $C$ is a given constant modelling, for instance, node's capacity. We consider 6 nodes, and different scenarios for nodes' activation and the presence of Byzantine nodes. Fig. \ref{fig} (top) illustrates the convergence of the perturbation scheme to $\nabla f(x)$ when no Byzantine nodes are present, (left) when one node perturbates the network at a time  (a.k.a. single perturbation), and (right) when all possible partitions of $\mathcal{I}$ are considered for nodes' activation (a.k.a. simultaneous perturbation). Fig. \ref{fig} (bottom) shows the results for same 6 nodes, but among which 2 are Byzantine, for the single-perturbation case (left) and simultaneous perturbation (right). The results over 10 simulations and the average are shown. The algorithm approaches the theoretical value of the gradient, while activation scheme  and presence of Byzantine nodes have an incidence in the results.

\begin{figure}[t]
\centering
\includegraphics[width=0.46\columnwidth]{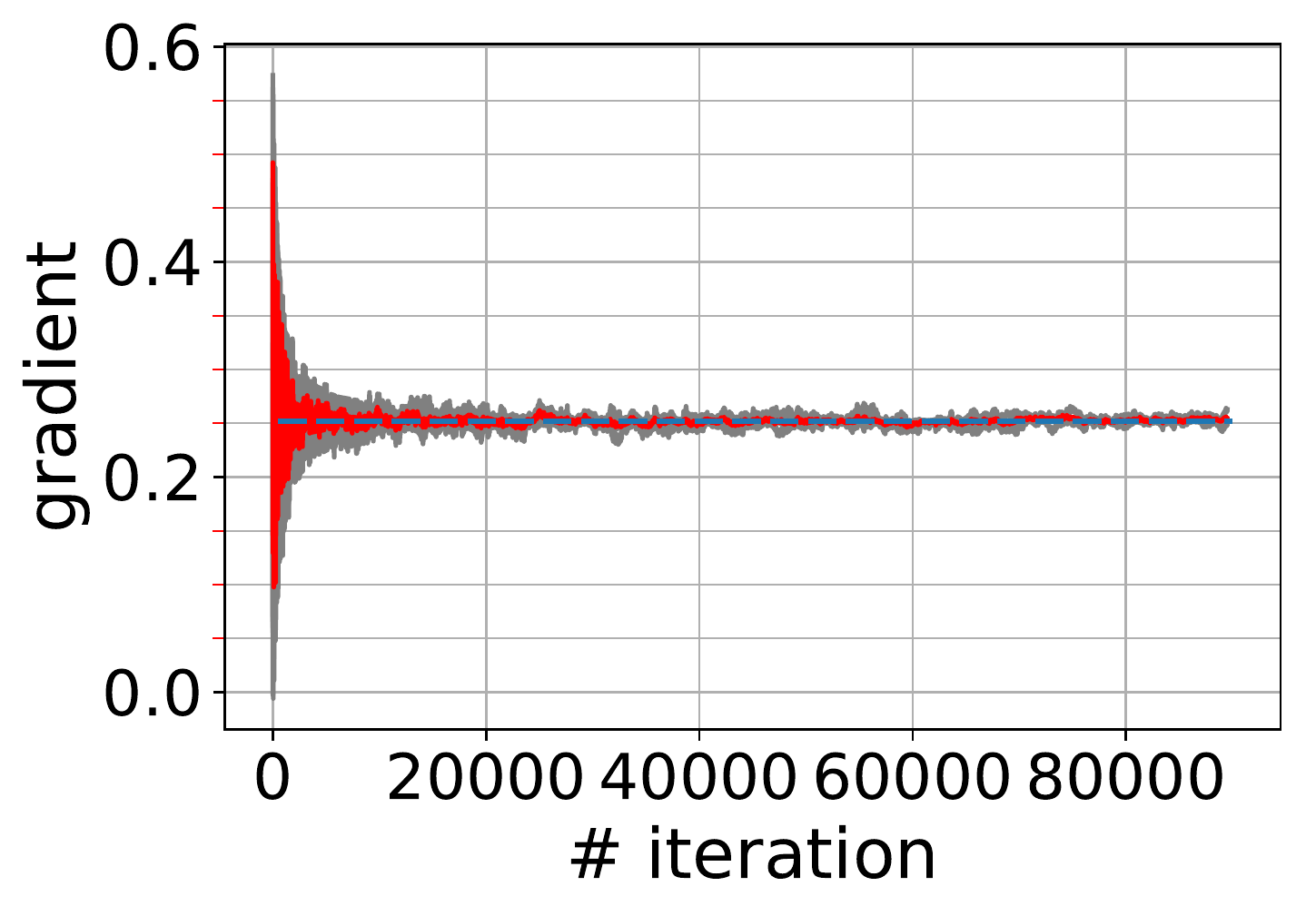}
\includegraphics[width=0.5\columnwidth]{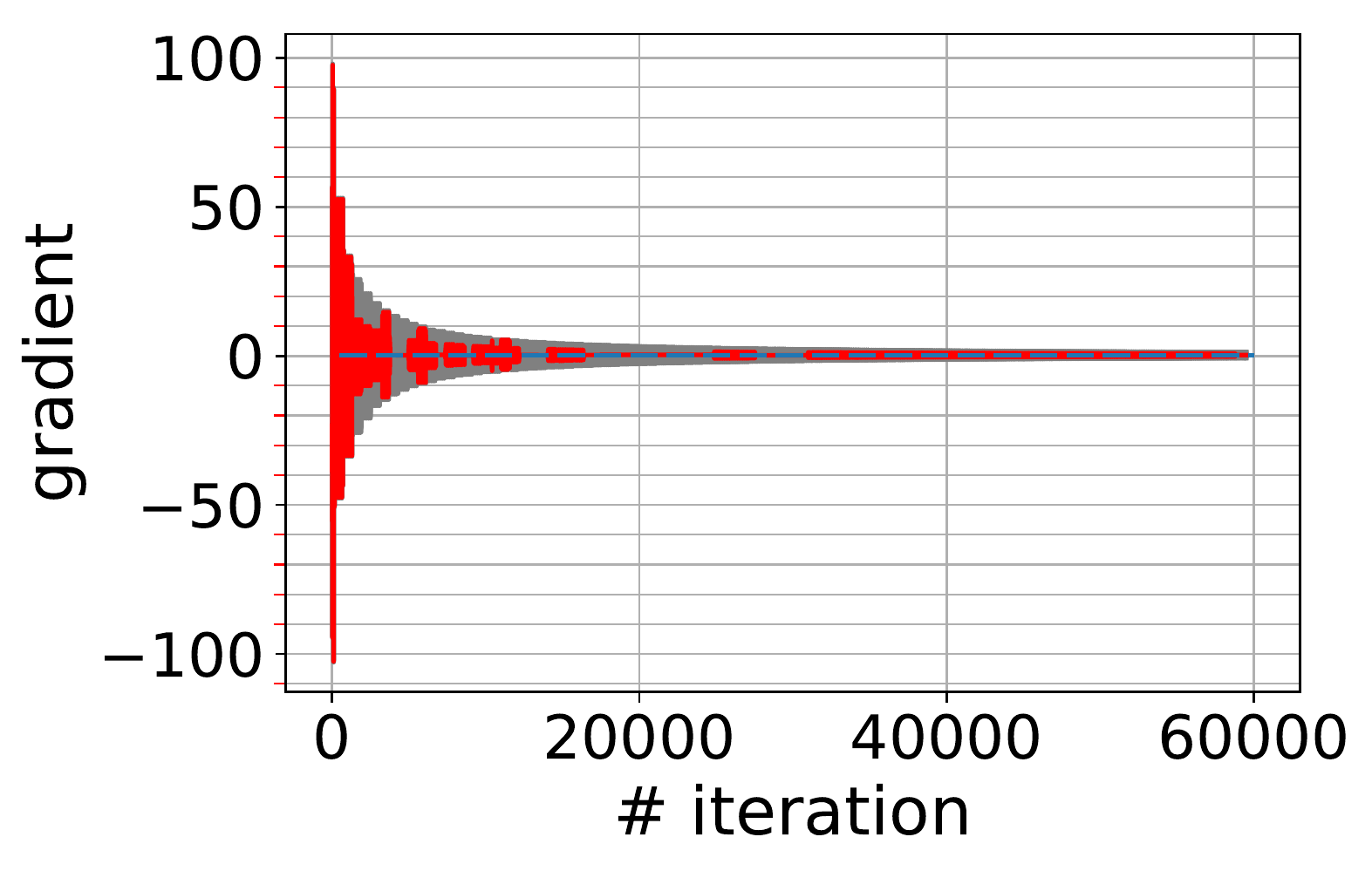}\\
\includegraphics[width=0.46\columnwidth]{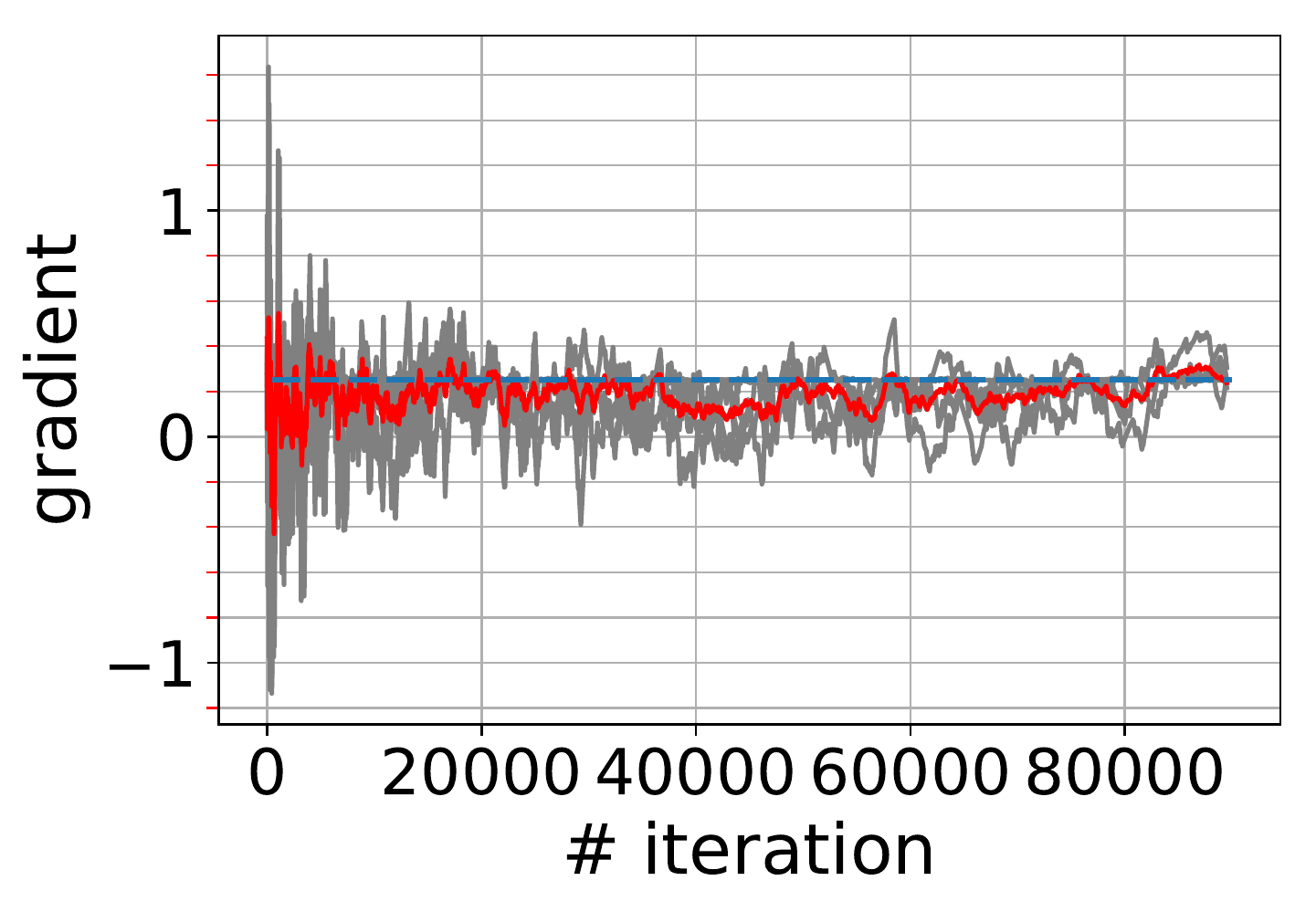}\
\includegraphics[width=0.5\columnwidth]{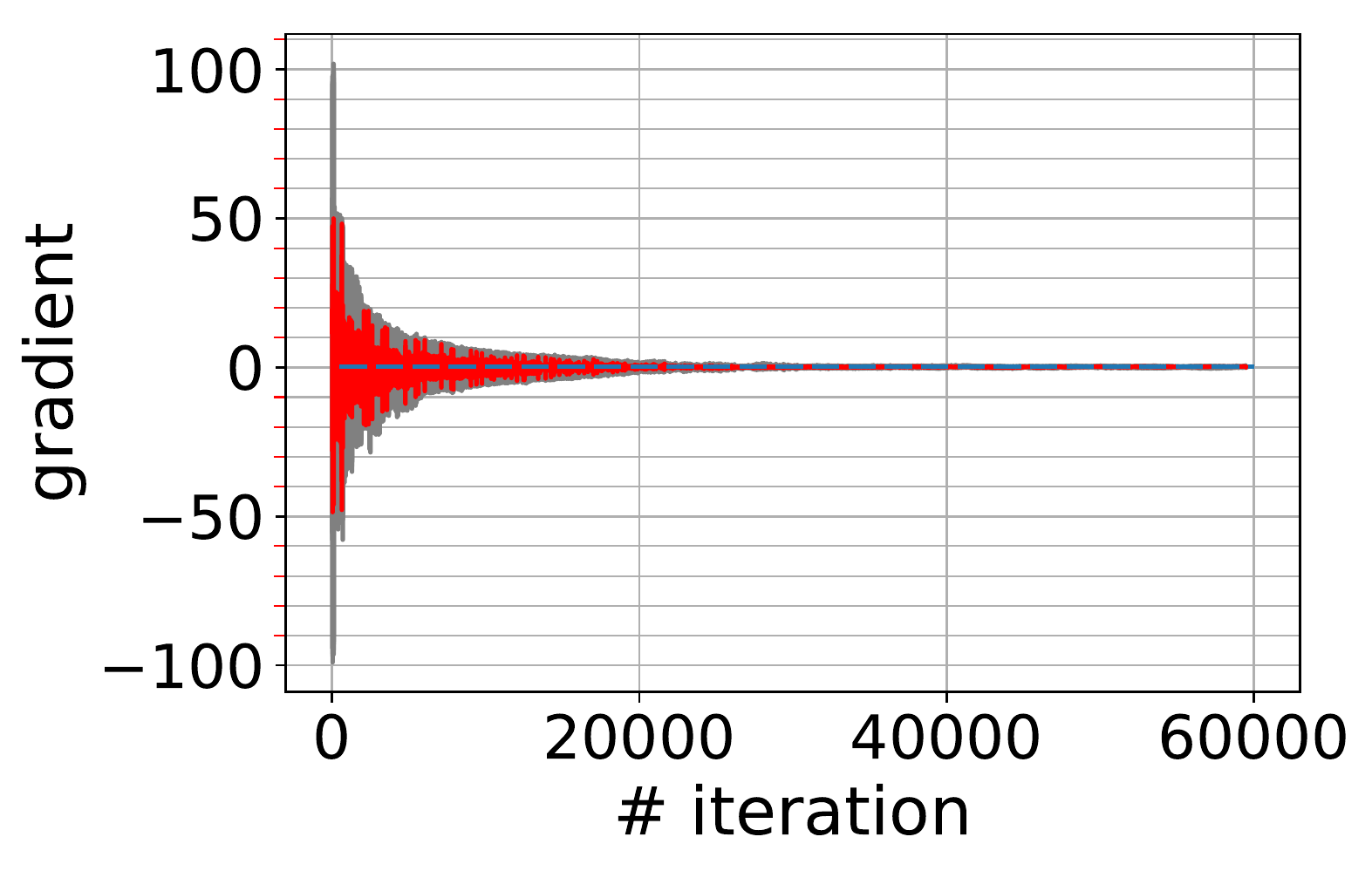}
\caption{Gradient estimation for 6 nodes. No Byzantine nodes (top) and 2 Byzantine nodes (bottom), with single (left) and simultaneous activation (right).}
\label{fig}
\end{figure}

\section{Conclusion}\label{sec:conclusion}
We have proposed a decentralized algorithm to estimate the gradient of a function describing a black-box model. Such algorithm is robust to the presence of Byzantine nodes  and takes into account their asynchronous behaviour. The solution is based on a two-timescale stochastic approximation and secure estimation. 
Future work will analyse convergence speed and optimize nodes activation.

\bibliographystyle{splncs04}
\bibliography{references.bib}
\end{document}